\def\year{2019}\relax
\documentclass[letterpaper]{article} 
\usepackage{aaai19}  
\usepackage{times}  
\usepackage{helvet}  
\usepackage{courier}  
\usepackage{url}  
\usepackage{graphicx}  
\usepackage[title]{appendix}
\usepackage{relsize}
\usepackage{amsmath}
\usepackage{amsfonts}
\usepackage{subcaption}
\usepackage{amsthm}
\usepackage{comment}
\newtheorem{theorem}{Theorem}

\newtheorem{lemma}[theorem]{Lemma}
\theoremstyle{definition}
\newtheorem{definition}{Definition}
\DeclareMathOperator*{\argmax}{arg\,max}

\usepackage{array}
\usepackage{algorithm}
\usepackage{algorithmicx}
\usepackage[noend]{algpseudocode}

\usepackage{amsmath}               
{
	\theoremstyle{plain}
	
}

\title{End-to-End Safe Reinforcement Learning through Barrier Functions for Safety-Critical Continuous Control Tasks}

\author{
	Richard Cheng,\textsuperscript{\rm 1}
	G\'abor Orosz,\textsuperscript{\rm 2}
	Richard M. Murray,\textsuperscript{\rm 1}
	Joel W. Burdick,\textsuperscript{\rm 1}\\
	\textsuperscript{\rm 1}California Institute of Technology,
	\textsuperscript{\rm 2}University of Michigan, Ann Arbor \\
}

\begin{document}
\maketitle

\begin{abstract}

Reinforcement Learning (RL) algorithms have found limited success beyond simulated applications, and one main reason is the absence of safety guarantees \textit{during} the learning process. Real world systems would realistically fail or break before an optimal controller can be learned. To address this issue, we propose a controller architecture that combines (1) a model-free RL-based controller with (2) model-based controllers utilizing control barrier functions (CBFs) and (3) on-line learning of the unknown system dynamics, in order to ensure safety during learning. Our general framework leverages the success of RL algorithms to learn high-performance controllers, while the CBF-based controllers both {\em guarantee} safety and \textit{guide} the learning process by constraining the set of explorable polices. We utilize Gaussian Processes (GPs) to model the system dynamics and its uncertainties.

Our novel controller synthesis algorithm, RL-CBF, guarantees safety with high probability during the learning process, regardless of the RL algorithm used, and demonstrates greater policy exploration efficiency. We test our algorithm on (1) control of an inverted pendulum and (2) autonomous car-following with wireless vehicle-to-vehicle communication, and show that our algorithm attains much greater sample efficiency in learning than other state-of-the-art algorithms \textit{and} maintains safety during the entire learning process. 

\end{abstract}

\section{Introduction}

Reinforcement learning (RL) focuses on finding an agent's policy (i.e. controller) that maximizes a long-term reward. It does this by repeatedly observing the agent's state, taking an action (according to a current policy), and receiving a reward.  Over time, the agent modifies its policy to maximize its long-term reward. This method has been successfully applied to continuous control tasks \cite{Duan2016,Lillicrap2015} where controllers have learned to stabilize complex robots (after many policy iterations).

However, since RL focuses on maximizing the long-term reward, it is likely to explore unsafe behaviors during the learning process. This feature is problematic for any RL algorithm that will be deployed on hardware, as unsafe learning policies could damage the hardware or bring harm to a human. As a result, most success in the use of RL for control of physical systems has been limited to simulations, where many failed iterations can occur before success.

Safe RL tries to learn a policy that maximizes the expected return, while also ensuring (or encouraging) the satisfaction of some safety constraints \cite{Garcia2015}. Previous approaches to safe reinforcement learning include \textit{reward-shaping}, \textit{policy optimization with constraints} \cite{Gaskett2003a,Moldovan2012,Achiam2017,Wachi2018}, or \textit{teacher advice} \cite{Abbeel2004,Abbeel2010,Tang2010}. However, these model-free approaches do not \textit{guarantee} safety during learning -- safety is only \textit{approximately} guaranteed after a sufficient learning period. The fundamental issue is that without a model, safety must be learned through environmental interactions, which means it may be violated during initial learning interactions.

Model-based approaches have utilized Lyapunov-based methods or model predictive control to guarantee safety under system dynamics during learning \cite{Wang2017,Berkenkamp2017,Chow2018,Ohnishi2018,Koller2018}, but they do not address the issue of exploration and performance optimization. Other works guarantee safety by switching between backup controllers \cite{Perkins2003,Mannucci2018}, though this overly constrains policy exploration. 

We draw inspiration from recent work that has incorporated model information into model-free RL algorithms to ensure safety during exploration \cite{Fisac2018,Li2018,Gillula2012}. However, these approaches utilize backup safety controllers that do not guide the learning process (limiting exploration efficiency).

This paper develops a framework for integrating existing model-free RL algorithms with \textit{control barrier functions} (CBFs) to guarantee safety and improve exploration efficiency in RL, even with uncertain model information. The CBFs require a (potentially poor) nominal dynamics model, but can ensure online safety of nonlinear systems during the entire learning process \textit{and} help the RL algorithm efficiently search the policy space. This methodology effectively constrains the policy exploration process to a set of safe polices defined by the CBF. An on-line process learns the governing dynamical system over time, which allows the CBF controller to adapt and become less conservative over time. This general framework allows us to utilize \textit{any} model-free RL algorithm to learn a controller, with the CBF controller guiding policy exploration and ensuring safety. 

Using this framework, we develop an efficient algorithm for controller synthesis, RL-CBF, with guarantees on safety (remaining within a safe set) and performance (reward-maximization). To test this approach, we integrated two model-free RL algorithms -- \textit{trust region policy optimization} (TRPO) \cite{Schulman2015} and \textit{deep deterministic policy gradients} (DDPG) \cite{Lillicrap2015} -- with the CBF controllers and dynamical model learning. We tested the algorithms on two nonlinear control problems: (1) balancing of an inverted pendulum, and (2) autonomous car following with wireless vehicle-to-vehicle communication. For both tasks, our algorithm efficiently learned a high-performance controller while maintaining safety throughout the learning process. Furthermore, it learned faster than comparable RL algorithms due to inclusion of a model learning process, which constrains the space of explorable policies and guides the exploration process. 

Our main contributions are: (1) we develop the first algorithm that integrates CBF-based controllers with model-free RL to achieve end-to-end safe RL for nonlinear control systems, and (2) we show improved learning efficiency by guiding the policy exploration with barrier functions.

\label{Introduction}

\section{Preliminaries}

Consider an infinite-horizon discounted Markov decision process (MDP) with control-affine, deterministic dynamics (a good assumption when dealing with robotic systems), defined by the tuple $(S, A, f, g, d, r, \rho_0, \gamma)$, where $S$ is a set of states, $A$ is a set of actions, $f : S \rightarrow S$ is the nominal unactuated dynamics, $g : S \rightarrow \mathbb{R}^{n,m}$ is the nominal actuated dynamics, and $d : S \rightarrow S$ is the \textit{unknown} system dynamics. The time evolution of the system is given by
\begin{equation}
s_{t+1} = f(s_t) + g(s_t) a_t + d(s_t) ,
\label{eq:transition_dynamics_2}
\end{equation}

\noindent
where $s_t \in S$, $a_t \in A$, $f$ and $g$ compose a known nominal model of the dynamics, and $d$ represents the unknown model. In practice, the nominal model may be quite bad (e.g. a robot model that ignores friction and compliance), and we must learn a much better dynamic model through data.

Furthermore $r : S \times A \rightarrow \mathbb{R}$ is the reward function, $\rho_0 : S \rightarrow \mathbb{R}$ is the distribution of the initial state $s_0$, and $\gamma \in (0,1)$ is the discount factor.

\subsection{Reinforcement Learning}

Let $\pi(a|s)$ denote a stochastic control policy $\pi : S \times A \rightarrow [0,1]$ that maps states to distributions over actions, and let $J(\pi)$ denote the policy's expected discounted reward:
\begin{equation}
 J(\pi) = \mathbb{E}_{\tau \sim \pi} [\sum_{t=0}^{\infty} \gamma^t r(s_t) ] .
\label{eq:RL_cost}
\end{equation}
Here $\tau \sim \pi$ is a trajectory $\tau = \{s_t, a_t, ..., s_{t+n}, a_{t+n} \}$ where the actions are sampled from policy $\pi(a|s)$. We use the standard definitions for the value function $V_{\pi}$, action-value function $Q_{\pi}$, and advantage function, $A_{\pi}$ below:
\begin{equation*}
\begin{split}
& Q_{\pi}(s_t, a_t) = \mathbb{E}_{s_{t+1},a_{t+1},...} \Big[ \sum_{l=0}^{\infty} \gamma^l r(s_{t+l}, a_{t+l}) \Big], \\
& V_{\pi}(s_t) = \mathbb{E}_{a_t, s_{t+1},a_{t+1},...} \Big[ \sum_{l=0}^{\infty} \gamma^l r(s_{t+l},a_{t+l}) \Big],
\end{split}
\label{eq:value_function_1}
\end{equation*}

\begin{equation}
\begin{split}
& A_{\pi}(s_t, a_t) = Q_{\pi}(s_t, a_t) - V_{\pi}(s_t),  \\
\end{split}
\label{eq:value_function}
\end{equation}

\noindent
where actions $a_i$ are drawn from distribution $a_i \sim \pi(a|s_i)$.

Most policy optimization RL algorithms attempt to maximize long-term reward $J(\pi)$ using (a) policy iteration methods \cite{Bertsekas2005}, (b) derivative-free optimization methods that optimize the return as a function of policy parameters \cite{Fu2005}, or (c) policy gradient methods \cite{Peters2008,Silver2014}. \textit{Any} of these methods can be rendered end-to-end safe using the RL-CBF control framework proposed in this work. However, we will focus mainly on policy gradient methods, due to their good performance on continuous control problems.

\label{RL_prelim}

\subsubsection{Policy Gradient-Based RL}

Policy gradient methods estimate the gradient of the expected return $J(\pi)$ with respect to the policy based on sampled trajectories. They then optimize the policy using gradient ascent, allowing modification of the control law at episodic intervals. The DDPG and TRPO algorithms are examples of policy gradient methods, which we will use as benchmarks.  

DDPG is an off-policy actor-critic method that computes the policy gradient based on sampled trajectories and an estimate of the action-value function. It alternately updates the action-value function and the policy as it samples more and more trajectories. 

TRPO is an on-policy policy gradient method that maximizes a surrogate loss function, which serves as an approximate lower bound on the true loss function. It also ensures that the next policy distribution is within a ``trust region''. More precisely, it approximates the optimal policy update by iteratively solving the optimization problem:
\begin{equation}
\begin{split}
& \pi_{i+1} = \argmax_{\pi} \sum_s \rho_{\pi_i}(s) \sum_a \pi(a | s) A_{\pi_i} (s,a)
\end{split}
\label{eq:trpo_11}
\end{equation}

\noindent 
such that the Kullback-Leibler divergence $D_{KL} (\pi_{i}, \pi_{i+1}) \leq \delta_{p}$. Here $\rho_{\pi_i}(s)$ represents the discounted visitation frequency of state $s$ under policy $\pi_i$, and $\delta_p$ is a constant defining the ``trust region''.

Though both DDPG and TRPO have learned good controllers on several benchmark problems, there is no guarantee of safety in these algorithms, \textit{nor any other model-free RL algorithm}. Therefore, our objective is to complement model-free RL controllers with model-based CBF controllers (using a potentially poor nominal model), which can both improve search efficiency and ensure safety.

\subsection{Gaussian Processes}

We use Gaussian process (GP) models to estimate the unknown system dynamics, $d(s)$, from data. A Gaussian process is a nonparametric regression method for estimating functions \textit{and their uncertain distribution} from data \cite{Rasmussen2006}. It describes the evolving model of the uncertain dynamics, $d(s)$, by a mean estimate, $\mu_d(s)$, and the uncertainty, $\sigma_d^2(s)$, which allows for high probability confidence intervals on the function:
\begin{equation}
\begin{split}
& \mu_d(s) - k_{\delta} \sigma_d(s) \leq d(s) \leq \mu_d(s) + k_{\delta} \sigma_d(s), \\
\end{split}
\label{eq:GP_bounds}
\end{equation}

\noindent
with probability $(1-\delta)$ where $k_{\delta}$ is a design parameter that determines $\delta$ (e.g. $95\%$ confidence is achieved at $k_{\delta} = 2$). Therefore, by learning $\mu_d(s)$ and $\sigma_d(s)$ in tandem with the controller, we obtain high probability confidence intervals on the unknown dynamics, which adapt/shrink as we obtain more information (i.e. measurements) on the system.

A GP model is parameterized by a kernel function $k(s,s')$, which defines the similarity between any two states $s,s' \in S$. In order to make inferences on the unknown function $d(s)$, we need measurements, $\hat{d}(s)$, which are computed from measurements of ($s_t, a_t, s_{t+1}$) using the relation from Equation (\ref{eq:transition_dynamics_2}): $\hat{d}(s_t) = s_{t+1} - f(s_t) - g(s_t) a_t$. Since any finite number of data points form a multivariate normal distribution, we can obtain the posterior distribution of $d(s_{*})$ at any query state $s_{*} \in S$ by conditioning on the past measurements. Given $n$ measurements $y_n = [\hat{d}(s_1), \hat{d}(s_2), ..., \hat{d}(s_n)]$ subject to independent Gaussian noise $\nu_{noise} \sim \mathcal{N}(0,\sigma_{noise}^2)$, the mean $\mu_d(s_*)$ and variance $\sigma_d^2(s_*)$ at the query state, $s_*$, are calculated to be,
\begin{equation}
\begin{split}
& \mu_d(s_*) = k_*^T(s_*) (K + \sigma_{noise}^2 I)^{-1} y_n , \\
& \sigma_d^2(s_*) = k(s_*, s_*) - k_*^T(s_*) (K + \sigma_{noise}^2 I)^{-1} k_*(s_*) ,
\end{split}
\label{eq:GP_model}
\end{equation}

\noindent
where $K_{i,j} = k(s_i, s_j)$ is the kernel matrix, and $k_* = [k(s_1, s_*), k(s_2, s_*), ..., k(s_n, s_*)]$. As we collect more data, $\mu_d(s)$ becomes a better estimate of $d(s)$, and the uncertainty, $\sigma_d^2(s)$, of the dynamics decreases.

We note that in applications with large amounts of data, training the GP becomes problematic since computing the matrix inverse in Equation (\ref{eq:GP_model}) scales poorly ($N^3$ in the number of data points). There are several methods to alleviate this issue, such as using sparse inducing inputs or local GPs \cite{Snelson2007,Nguyen-Tuong2009}. In fact, our framework can use any model approximation method that provides quantifiable uncertainty bounds (e.g. neural networks with dropout). However, we bypass this issue in this work by batch training the GP model with only the latest batch of $\approx1000$ data points.

\label{GP_prelim}

\subsection{Control Barrier Functions}

Consider an arbitrary safe set, $\mathcal{C}$, defined by the super-level set of a continuously differentiable function $h: \mathbb{R}^n \rightarrow \mathbb{R}$,
\begin{equation}
\begin{split}
&  \mathcal{C}: \{s \in \mathbb{R}^n: h(s) \geq 0 \}.
\end{split}
\label{eq:safe_set}
\end{equation}

To maintain safety during the learning process, the system state must always remain within the safe set $\mathcal{C}$ (i.e. the set $\mathcal{C}$ is \textit{forward invariant}). Examples include keeping a manipulator within a given workspace, or ensuring that a quadcopter avoids obstacles. Essentially, the learning algorithm should learn/explore only in set $\mathcal{C}$.

\textit{Control barrier functions} utilize a Lyapunov-like argument to provide a sufficient condition for ensuring forward invariance of the safe set $\mathcal{C}$ under controlled dynamics. Therefore, barrier functions are a natural tool to enforce safety throughout the learning process, and can be used to synthesize \textit{safe} controllers for our systems.

\begin{definition}
Given a set $\mathcal{C} \in \mathbb{R}^n$ defined by (\ref{eq:safe_set}), the continuously differentiable function $h: \mathbb{R}^n \rightarrow \mathbb{R}$ is a \textit{discrete-time control barrier function} (CBF) for dynamical system (\ref{eq:transition_dynamics_2}) if there exists $\eta \in [0,1]$  such that for all $s_t \in C$, 

\begin{equation}
\begin{split}
& \sup_{a_t \in A} \Big[ h\Big(f(s_t) + g(s_t) a_t + d(s_t) \Big) + (\eta - 1) h(s_t) \Big] \geq 0 ,\\
\end{split}
\label{eq:discrete_CBF}
\end{equation}
\label{def:CBF}
\end{definition} 

\noindent
where $\eta$ represents how strongly the barrier function ``pushes'' the state inwards within the safe set (if $\eta = 0$, the barrier condition simplifies to the Lyapunov condition).

The existence of a CBF implies that there exists a \textit{deterministic} controller $u^{CBF}: S \rightarrow A$ such that the set $\mathcal{C}$ is forward invariant for system (\ref{eq:transition_dynamics_2}) \cite{agrawal2017,ames2017}. In other words, if condition (\ref{eq:discrete_CBF}) is satisfied for all $s \in \mathcal{C}$, then the set $\mathcal{C}$ is rendered forward invariant. Our goal is to find a controller, $u^{CBF}$, that satisfies condition (\ref{eq:discrete_CBF}), so that safety is certified.

For this paper, we restrict our attention to \textit{affine} barrier functions of form $h = p^T s + q$,  ($p \in \mathbb{R}^n, q \in \mathbb{R}$), though our methodology could support more general barrier functions. This restriction means the set $\mathcal{C}$ is composed of intersecting half spaces (i.e. polytopes). 

Before we can formulate a tractable optimization problem that satisfies condition (\ref{eq:discrete_CBF}), we must have an estimate for $d(s)$. We use an updating GP model to estimate the mean and variance of the function, $\mu_d(s)$ and $\sigma_d^2(s)$, from measurement data. From equation (\ref{eq:GP_bounds}), we know that $|\mu_d(s) - d(s)| \leq k_{\delta} \sigma_d(s) $ with probability $(1-\delta)$. Therefore, we can reformulate the CBF condition (\ref{eq:discrete_CBF}) into the following quadratic program (QP) that can be efficiently solved at each time step:

\begin{equation}
\begin{aligned}
(a_t, \epsilon) = ~ & \underset{a_t,\epsilon}{\text{argmin}}
& & \| a_t \|_2 + K_{\epsilon} \epsilon \\
& ~ \text{s.t.}
& &   p^T f(s_t) + p^T g(s_t) a_t + p^T \mu_d(s_t) - \\
& & &  k_{\delta}|p |^T \sigma_d(s_t) + q \geq (1-\eta) h(s_t) - \epsilon \\
& & & a^i_{low} \leq a_t^i \leq a^i_{high} ~~ \textnormal{for} ~ i = 1,...,M
\end{aligned}
\label{eq:optimization_GP}
\end{equation}

\noindent
where $\epsilon$ is a slack variable in the safety condition, $K_{\epsilon}$ is a large constant that penalizes safety violations, and $|p|$ denotes the element-wise absolute value of the vector $p$. The optimization is not sensitive to the $K_{\epsilon}$ parameter as long as it is very large (e.g. $10^{12}$), such that safety constraint violations are heavily penalized. The last constraint on $a_t^{i}$ encodes actuator constraints. The solution to this optimization problem (\ref{eq:optimization_GP}) enforces the safety condition (\ref{eq:discrete_CBF}) as best as possible with minimum control effort, even with uncertain dynamics. Accounting for the dynamics uncertainty through GP models allows us to certify system safety, even with a poor nominal model.

Let us define the set $\mathcal{C}_{\epsilon}: \{ s \in \mathbb{R}^n : h(s) \geq -\frac{\epsilon}{\eta} \}$. Then we can prove the following lemma. 

\begin{lemma}
	For dynamical system (\ref{eq:transition_dynamics_2}), if there exists a solution to (\ref{eq:optimization_GP}) for all $s \in \mathcal{C}$ with $\epsilon=0$, then the controller derived from (\ref{eq:optimization_GP}) renders set $\mathcal{C}$ forward invariant with probability $(1-\delta)$. 
	
	However, suppose there exists $s \in \mathcal{C}$ such that (\ref{eq:optimization_GP}) has solution with $\epsilon = \epsilon^{max} > 0$. If for all $s \in \mathcal{C}_{\epsilon}$, the solution to (\ref{eq:optimization_GP}) satisfies $\epsilon \leq \epsilon^{max}$, then the larger set $\mathcal{C}_{\epsilon}$ is forward invariant with probability $(1-\delta)$.
	\label{lemma:lemma_barrier_2}
\end{lemma}

\begin{proof}
	The first part of the lemma follows directly from Definition \ref{def:CBF} and the probabilistic bounds on the uncertainty obtained from GPs shown in equation (\ref{eq:GP_bounds}).
	
	For the second part, the property of GPs in equation (\ref{eq:GP_bounds}) implies that with probability $(1-\delta)$, the following inequality is satisfied under the system dynamics (\ref{eq:transition_dynamics_2}):
	\begin{equation}
	\begin{split}
	& h(s_{t+1}) \geq p^T \Big( f(s_t) + g(s_t) a_t + \mu_d(s_t) \Big) - \\
	& ~~~~~~~~~~~~~~~~~~~~~~~~~~~~~~~~~~~~ k_{\delta} | p |^T \sigma_d(s_t) + q.
	\end{split}
	\label{eq:barrier_proof_1}
	\end{equation}
	
	Therefore, the constraint in problem (\ref{eq:optimization_GP}) ensures that: 
	\begin{equation}
	\begin{split}
	& h(s_{t+1}) \geq (1 - \eta ) h(s_t) - \epsilon , \\
	& p^T s_{t+1} + q \geq (1 - \eta) (p^T s_t + q) -  \epsilon , \\
	& p^T s_{t+1} + q + \frac{\epsilon}{\eta} \geq (1 - \eta) (p^T s_t + q + \frac{\epsilon}{\eta}) .
	\end{split}
	\label{eq:barrier_proof_2}
	\end{equation}
	
	Define $h_{\epsilon}(s) = q + \frac{\epsilon}{\eta} + p^T s$, so that (\ref{eq:barrier_proof_2}) simplifies to 
	\begin{equation}
	\begin{split}
	& h_{\epsilon}(s_{t+1}) \geq (1-\eta) h_{\epsilon}(s_t) . \\
	\end{split}
	\label{eq:barrier_proof_3}
	\end{equation}
	
	By Definition \ref{def:CBF}, the set $\mathcal{C}_{\epsilon}$ defined by $h_{\epsilon}(s) = h(s) + \frac{\epsilon}{\eta} \geq 0$ is forward invariant under system dynamics (\ref{eq:transition_dynamics_2}).
\end{proof}

The CBF controllers that solve (\ref{eq:optimization_GP}) provide deterministic control laws, $u^{CBF}(s)$ that naturally encode safety; they provide the \textit{minimal} control intervention that maintains safety \textit{or} provide graceful degradation (a small deviation from the safe set) when safety cannot be enforced (e.g. due to actuation constraints). Furthermore, even with dynamics uncertainty, we can make high-probability statements about system safety using GP models with CBFs.

Note that one can easily combine multiple CBF constraints in problem (\ref{eq:optimization_GP}) to define polytopic safe regions.

\section{CBF-Based Compensating Control with Reinforcement Learning}

To illustrate our framework, we first propose the \textit{suboptimal} controller in equation (\ref{eq:framework}), which combines a model-free RL-based controller (parameterized by $\theta_k$) and a CBF-based controller in the architecture shown in Figure \ref{fig:architecture_1}.
\begin{equation}
\begin{split}
u_k(s) = u_{\theta_k}^{RL}(s) + u_k^{CBF}(s, u_{\theta_k}^{RL}).
\end{split}
\label{eq:framework}
\end{equation}

The concept is akin to shielded RL \cite{Alshiekh2017,Fisac2018}, since the CBF controller compensates for the RL controller to ensure safety, but does not guide exploration of the RL algorithm. The next section will extend the CBF controller to improve RL policy exploration.

Note that since the RL policy $\pi_{\theta_k}^{RL}(a|s)$ is stochastic (see Preliminaries section on RL), the controller $u_{\theta_k}^{RL}(s)$ represents the realization (i.e. sampled control action) of the stochastic policy $\pi_{\theta_k}^{RL}(a|s)$ after policy iteration $k$.

\begin{figure}[!h]
	\centering
	\begin{subfigure}[b]{0.5\textwidth}
		\includegraphics[width=0.9\linewidth]{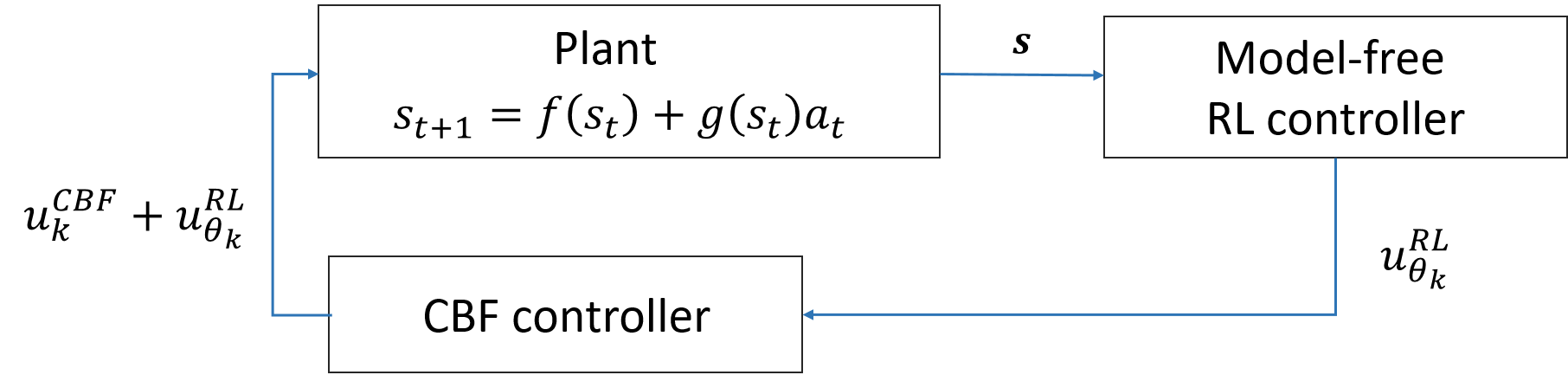}
		\caption{}
		\label{fig:architecture_1} 
	\end{subfigure}
	
	\begin{subfigure}[b]{0.5\textwidth}
		\includegraphics[width=0.9\linewidth]{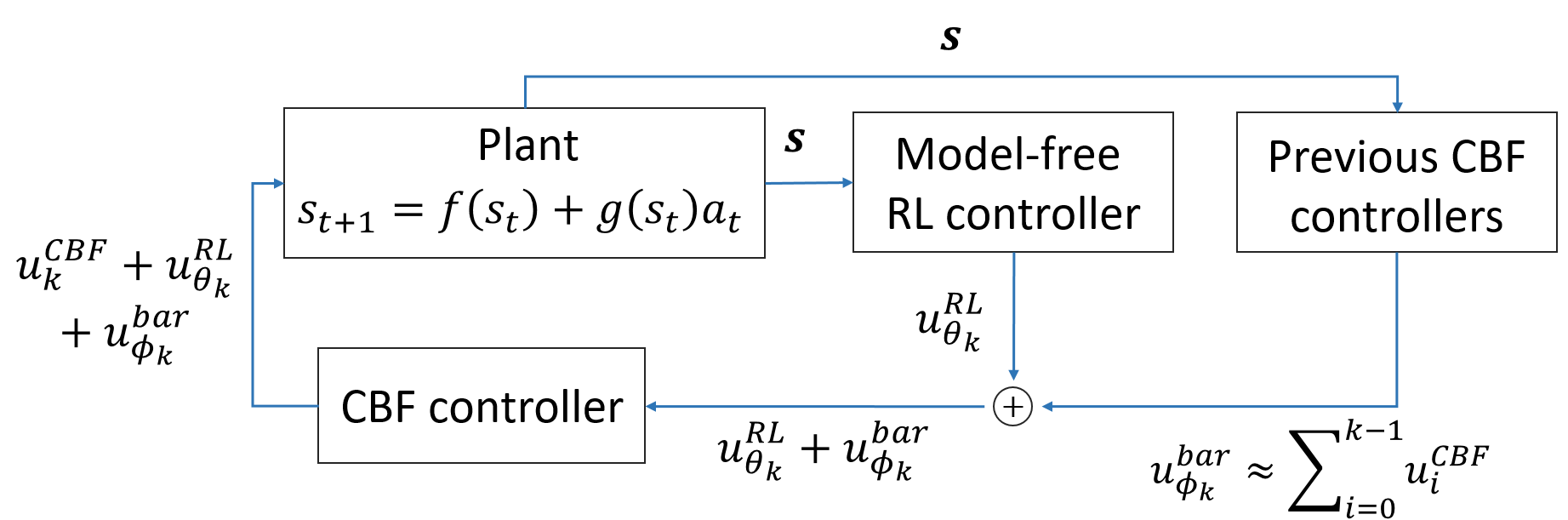}
		\caption{}
		\label{fig:architecture_2}
	\end{subfigure}
	\caption{Control architecture combining model-free RL controller with model-based CBF to guarantee safety. (a) Initial architecture that uses CBF to \textit{compensate} for unsafe control actions, but does not guide learning and exploration. (b) Architecture that uses CBF to guide exploration and learning, as well as ensure safety.}
	\label{fig:architecture}
\end{figure}

The model-free RL controller, $u^{RL}_{\theta_k}(s)$ proposes a control action that attempts to optimize long-term reward, but may be unsafe. Before deploying the RL controller, a CBF controller $u_k^{CBF}(s, u^{RL}_{\theta_k})$ filters the proposed control action and provides the \textit{minimum} control intervention needed to ensure that the overall controller, $u_k(s)$, keeps the system state within the safe set. Essentially, the CBF controller, $u_k^{CBF}(s, u^{RL}_{\theta_k})$ ``projects'' the RL controller $u_{\theta_k}^{RL}(s)$ into the set of safe policies. In the case of an autonomous car, this action may enforce a safe distance between nearby cars, regardless of the action proposed by the RL controller.

The CBF controller $u_k^{CBF}(s, u_{\theta_k}^{RL})$, which depends on the RL control, is defined by the following QP that can be efficiently solved at each time step:

\begin{equation}
\begin{aligned}
(a_t, \epsilon) = ~ & ~ \underset{a_t,\epsilon}{\text{argmin}} ~ \| a_t \|_2 + K_{\epsilon} \epsilon \\
\text{s.t.} & ~ ~ p^T f(s_t) + p^T g(s_t) \Big( u^{RL}_{\theta_k}(s_t) + a_t \Big) + p^T \mu_d(s_t) \\
~~~~~~~~~~~ & ~~~~~~~ - k_{\delta} |p|^T \sigma_d(s_t) + q \geq (1 - \eta ) h(s_t) - \epsilon \\
~~~~~~~~ & a^i_{low} \leq a_t^i + u^{RL (i)}_{\theta_k}(s_t) \leq a^i_{high} ~ \textnormal{for} ~ i = 1,...,M
\end{aligned}
\label{eq:barrier_control}
\end{equation}

\noindent
The last constraint in (\ref{eq:barrier_control}) incorporates possible actuator limits of the system. 

We must make clear the important distinction between the indexes $t$ and $k$. Note that $t$ indexes timesteps \textit{within} each policy iteration or trial, whereas $k$ indexes the policy iterations (which contain trajectories with several timesteps). The CBF controller updates throughout the task (computed at each time step, $t$), whereas the RL policy and GP model update at episodic policy iteration intervals indexed by $k$. 

Let $\epsilon^{max} = \max_{s \in \mathcal{C}}$ $\mathlarger\epsilon ~~ \textnormal{from (\ref{eq:barrier_control})}$ represent the largest violation of the barrier condition (i.e. potential deviation from the safe set) for any $s \in \mathcal{C}$. Lemma \ref{lemma:lemma_barrier_2} extends to the modified optimization problem (\ref{eq:barrier_control}), implying that $u_k = u_{\theta_k}^{RL} + u_k^{CBF}$ satisfies the barrier certificate inequality (up to $\epsilon^{max}$) that guarantees forward invariance of $\mathcal{C}$. Therefore, if there exists a solution to problem (\ref{eq:barrier_control}) such that $\epsilon^{max} = 0$, then controller (\ref{eq:framework}) renders the safe set $\mathcal{C}$ forward invariant with probability $(1-\delta)$. However if $\epsilon^{max} > 0$, but $\epsilon \leq \epsilon^{max}$ for all $s \in \mathcal{C}_{\epsilon}$, then the controller will render the set $\mathcal{C}_{\epsilon}$ forward invariant with probability $(1-\delta)$.

Intuitively, the RL controller provides a ``feedforward control'', and the CBF controller compensates with the minimum control necessary to render the safe set forward invariant. If such a control does not exist (e.g. due to torque constraints), then the CBF controller provides the control that keeps the state as close as possible to the safe set.

However, a significant issue is that controller (\ref{eq:framework}) ensures safety, but does not actively guide policy exploration of the overall controller. This is because the RL policy being updated around, $u_{\theta_k}^{RL}(s)$, is \textit{not} the policy deployed on the agent, $u_k(s)$. For example, suppose that in an autonomous driving task, the RL controller inadvertently proposes to collide with an obstacle. The CBF controller compensates to drive the car around the obstacle. The next learning iteration should update the policy around the safe deployed policy $u_k(s)$, rather than the unsafe policy $u_{\theta_k}^{RL}(s)$ (which would have led to an obstacle collision). However, the algorithm described in this section updates around the original policy, $u_{\theta_k}^{RL}(s)$, as illustrated in Figure \ref{fig:learning_process}a.

\section{CBF-Based Guiding Control with Reinforcement Learning} \label{barrier_guided}

In order to achieve safe \textit{and efficient} learning, we should learn from the deployed controller $u_k$, since it operates in the safe region $\mathcal{C}$, rather than learning around $u_{\theta_k}^{RL}$, which may operate in an unsafe, irrelevant area of state space. The \textit{RL-CBF} algorithm described below incorporates this goal.

Recall that $u_k, u_{\theta_k}^{RL}$ represent the realized controllers sampled from stochastic policies $\pi_k, \pi_{\theta_k}^{RL}$. Consider an initial RL-based controller $u_{\theta_0}^{RL}(s)$ (for iteration $k=0$). The CBF controller $u_0^{CBF}(s)$ is determined from (\ref{eq:barrier_control}) to obtain $u_0(s) = u_{\theta_0}^{RL}(s) + u_0^{CBF}(s)$. For every following policy iteration, let us define the overall controller to incorporate all previous CBF controllers, as in equation (\ref{eq:controllers_all}).

\begin{equation}
\begin{split}
&  u_k(s) = u_{\theta_k}^{RL}(s) + \sum_{j=0}^{k-1}u_j^{CBF}(s, u_{\theta_0}^{RL}, ..., u_{\theta_{j-1}}^{RL}) \\
&  ~~~~~~~~~~~ + u_k^{CBF}(s, u_{\theta_k}^{RL} + \sum_{j=0}^{k-1}u_j^{CBF}).
\end{split}
\label{eq:controllers_all}
\end{equation}

The dependence of controller (\ref{eq:controllers_all}) on all prior CBF controllers (see Figure \ref{fig:architecture_2}) is critical to enhancing learning efficiency. Defining the controller in this fashion leads to policy updates around the previously \textit{deployed} controller, which adds to the efficiency of the learning process by encouraging the policy to operate in desired areas of the state space. This idea is illustrated in Figure \ref{fig:learning_process}b. 

The intuition is that at iteration $k=0$, the RL policy \textit{proposed} actions $u_{\theta_0}^{RL}(s)$, but it \textit{took safe} actions $u_{\theta_0}^{RL}(s) + u_0^{CBF}(s)$. To update the policy based on the safe actions, the effective RL controller at the next iteration ($k=1$) should be $u_{\theta_1}^{RL}(s) + u_0^{CBF}(s)$, which is then filtered by the CBF controller $u_1^{CBF}(s)$ (i.e. $u_0^{CBF}(s)$ is now part of the RL controller). Across multiple policy iterations, we can consider $u_{\theta_k}^{RL}(s) + \sum_{j=0}^{k-1}u_j^{CBF}(s, u_{\theta_0}^{RL}, ..., u_{\theta_{j-1}}^{RL})$ to be the guided RL controller (proposing potentially unsafe actions), which is rendered safe by $u_k^{CBF}(s, u_{\theta_k}^{RL} + \sum_{j=0}^{k-1}u_j^{CBF})$.

\begin{figure}[!h]
\centering
\begin{subfigure}[b]{0.5\textwidth}
\includegraphics[width=0.9\columnwidth]{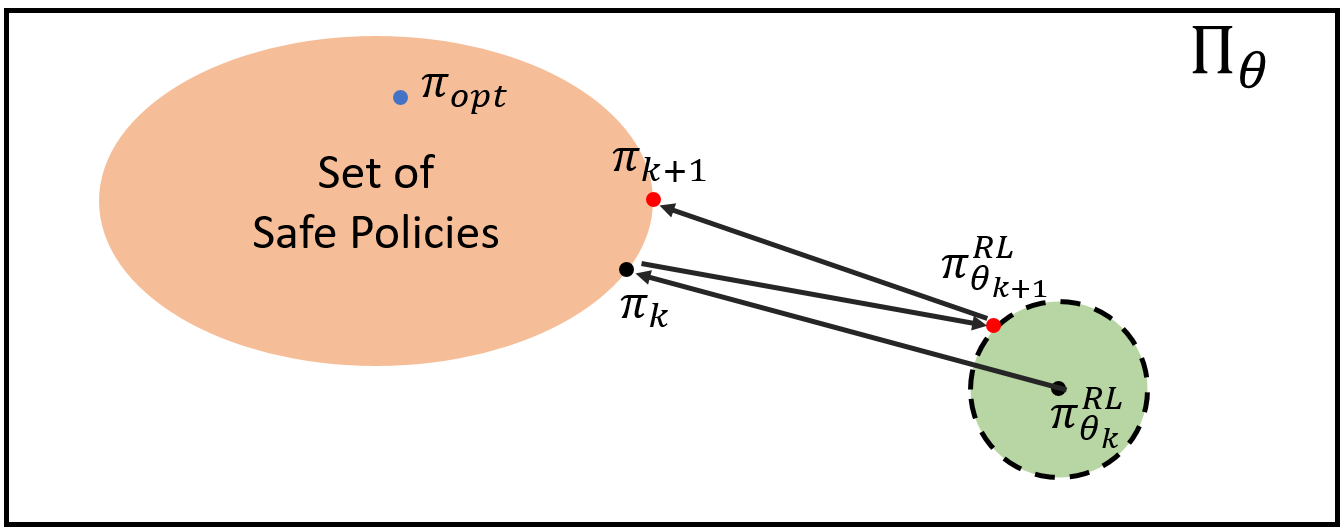}
\caption{}
\label{fig:learning_process_a} 
\end{subfigure}
\begin{subfigure}[b]{0.5\textwidth}
\includegraphics[width=0.9\columnwidth]{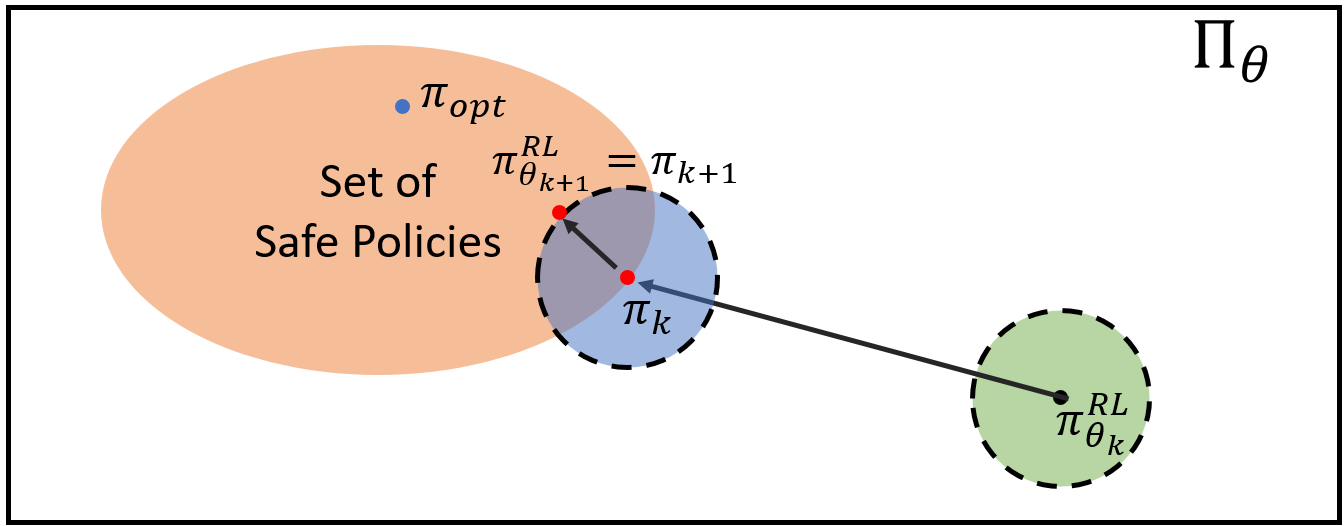}
\caption{}
\label{fig:learning_process_b}
\end{subfigure}
\caption{Illustration of policy iteration process, where we try to learn the optimal safe policy, $\pi_{opt}$. (a) Policy optimization with barrier-compensating controller. Next policy is updated around the previous RL controller, $\pi_{\theta_k}^{RL}$; (b) Policy optimization with barrier-guided controller. Next policy is updated around previous deployed controller, $\pi_k$.}
\label{fig:learning_process}
\end{figure}

To ensure safety after incorporating all prior CBF controllers, they must be included into the governing QP:

\begin{equation}
\begin{aligned}
& (a_t, \epsilon) = \underset{a_t,\epsilon}{\text{argmin}} ~ \| a_t \|_2 + K_{\epsilon} \epsilon \\
& ~ \text{s.t.} ~~~ p^T f(s_t) + p^T g(s_t) \Big( u_{\theta_k}^{RL}(s_t) + \sum_{j=0}^{k-1} u_j^{CBF}(s_t) + a_t \Big) \\
& ~~~~~~ + p^T \mu_d(s_t) - k_{\delta} |p|^T \sigma_d(s_t) + q  \geq (1 - \eta ) h(s_t) - \epsilon \\
& ~~~~~~~ a^i_{low} \leq a_t^i + u^{RL}(s_t) + \sum_{j=0}^{k-1} u_j^{CBF}(s_t) \leq a^i_{high}  \\
& ~~~~~~~~~~~~~~~~~~~~~~~~~~~~~~~~~~~~~~~~~~~~~~~~~~~~~~~~~~ \textnormal{for} ~ i = 1,...,M.
\end{aligned}
\label{eq:barrier_control_guide}
\end{equation}

\noindent
The solution to (\ref{eq:barrier_control_guide}) defines the CBF controller $u_k^{CBF}(s)$, which ensures safety by satisfying the barrier condition (\ref{eq:discrete_CBF}).

Let $\epsilon^{max} = \max_{s \in \mathcal{C}} \mathlarger\epsilon ~ \textnormal{from} ~ \textnormal{(\ref{eq:barrier_control_guide})}$ represent the largest violation of the barrier condition for any $s \in \mathcal{C}$.

\begin{theorem}	
	
	Using the control law $u_k(s)$ from (\ref{eq:controllers_all}), if there exists a solution to problem (\ref{eq:barrier_control_guide}) such that $\epsilon^{max} = 0$, then the safe set $\mathcal{C}$ is forward invariant with probability $(1-\delta)$. If $\epsilon^{max} > 0$, but the solution to problem (\ref{eq:barrier_control_guide}) satisfies $\epsilon \leq \epsilon^{max}$ for all $s \in \mathcal{C}_{\epsilon}$, then the controller will render the larger set $\mathcal{C}_{\epsilon}$ forward invariant with probability $(1-\delta)$.
	
	Furthermore, if we use TRPO for the RL algorithm, then the control law $u_k^{prop}(s) = u_k(s) - u_k^{CBF}(s)$ from (\ref{eq:controllers_all}) achieves the performance guarantee $J({\pi_{k}^{prop}}) \geq J({\pi_{k-1}}) - \frac{2 \lambda \gamma}{(1-\gamma)^2} \delta_{\pi}$, where $\lambda = \max_s | \mathbb{E}_{a \sim \pi_{k}^{prop}} [ A_{\pi_{k-1}}(s,a)]|$ and $\delta_{\pi}$ is chosen as in equation (\ref{eq:trpo_11}).
	\label{theorem:trpo_cbf}
\end{theorem}

\begin{proof}
	The first part of the theorem follows directly from Definition \ref{def:CBF} and Lemma \ref{lemma:lemma_barrier_2}. The only difference from Lemma \ref{lemma:lemma_barrier_2} is that the control includes the RL controller and all \textit{previous} CBF controllers ($u^{CBF}_0,...,u^{CBF}_{k-1}$). 
	
	The proof of the performance bound is given in the Appendix of this paper found at ~ \url{https://rcheng805.github.io/files/aaai2019.pdf}.
\end{proof}

RL-CBF provides high-probability safety guarantees during the learning process \textit{and} can maintain the performance guarantees of TRPO. If we have no uncertainty in the dynamics, then safety is guaranteed with probability 1.  Note that the performance guarantee in Theorem \ref{theorem:trpo_cbf} is for control law $u_k(s) - u_k^{CBF}(s)$, which is not the deployed controller, $u_k(s)$. However, this does not pose a significant issue, since $u_k^{CBF}(s)$ rapidly decays to 0 as we iterate. This is because the guided RL controller quickly learns to operate in the safe region, so the CBF controller $u_k^{CBF}(s)$ becomes inactive.

\section{Computationally Efficient Algorithm}

This section describes an efficient algorithm to implement the framework described above, since a naive approach would be too computationally expensive in many cases. To see this, recall the controller (\ref{eq:controllers_all}) we would ideally implement:
\begin{equation*}
\begin{split}
&  u_k(s) = u_{\theta_k}^{RL}(s) + \sum_{j=0}^{k-1}u_j^{CBF}(s, u_{\theta_0}^{RL}, ..., u_{\theta_{j-1}}^{RL}) \\
&  ~~~~~~~~~~~ + u_k^{CBF}(s, u_{\theta_k}^{RL} + \sum_{j=0}^{k-1}u_j^{CBF}).
\end{split}
\label{eq:controller_guided_1}
\end{equation*}

The first term may be represented by a neural network that is parameterized by $\theta_k$, which has a standard implementation. The third term is just a quadratic program with dependencies on the other terms; it does not pose a computational burden. \textit{However}, the summation in the 2nd term poses a challenge, since \textit{every} term in $\sum_{j=0}^{k-1}u_j^{CBF}(s, u_{\theta_0}^{RL}, ..., u_{\theta_{j-1}}^{RL})$ depends on a different previous RL controller $u_{\theta_j}^{RL}$. Therefore, we would need to store $k-1$ neural networks corresponding to each previous RL controller. In addition, we would have to solve $k-1$ separate QPs in sequence to evaluate each CBF controller. Such a brute-force implementation would be impractical .

To overcome this issue, we approximate $u^{bar}_{\phi_k}(s) \approx \sum_{j=0}^{k-1}u_j^{CBF}(s, u_{\theta_0}^{RL}, ..., u_{\theta_{j-1}}^{RL})$, where $u^{bar}_{\phi_k}$ is a feedforward neural network (MLP) parameterized by $\phi$. We chose a MLP since they have been shown to be powerful function approximators. Thus, at each policy iteration, we fit the MLP $u^{bar}_{\phi_k}(s)$ to data of $\sum_{j=0}^{k-1}u_j^{CBF}(s, u_{\theta_0}^{RL}, ..., u_{\theta_{j-1}}^{RL})$ collected from trajectories of the previous policy iteration. Then we obtain the controller:
  \begin{equation*}
  \begin{split}
  & \hspace*{-0.07cm} u_k(s) = u_{\theta_k}^{RL}(s) + u^{bar}_{\phi_k}(s) + u_k^{CBF}(s, u_{\theta_k}^{RL}+u^{bar}_{\phi_k}) .
  \end{split}
  \label{eq:controller_guided_2}
  \end{equation*}
  
  Note that even with this approximation, \textit{safety with probability $(1-\delta)$ is still guaranteed}. This is because the above approximation only affects the guided RL term $u_{\theta_k}^{RL}(s) + \sum_{j=0}^{k-1}u_j^{CBF}(s, u_{\theta_0}^{RL}, ..., u_{\theta_{j-1}}^{RL})$. The CBF controller $u_k^{CBF}(s, u_{\theta_k}^{RL}+u^{bar}_{\phi})$ still solves (\ref{eq:barrier_control_guide}), which provides the safety guarantees in Theorem \ref{theorem:trpo_cbf} by satisfying the CBF condition (\ref{eq:discrete_CBF}). Furthermore, we now have to store only two NNs and solve one QP for the controller. The tradeoff is that the performance guarantee in Theorem \ref{theorem:trpo_cbf} does not necessarily hold with this approximation. The algorithm is outlined in Algorithm \ref{alg:rl-cbf}.
  
  \begin{algorithm}[!h]
  	\caption{RL-CBF algorithm}\label{alg:rl-cbf}
  	\begin{algorithmic}[1]
  		\State Initialize RL Policy $\pi^{RL}_0$, state $s_0 \sim \rho_0$, 
  		\Statex ~~~~~~~~ measurement array $\hat{D}$, action array $\hat{A}$
  		\For {$t = 1,\ldots,T$}
  		\State Sample (but do not deploy) control $u^{RL}_{\theta_0} (s_t)$
  		\State Solve for $u^{CBF}_0(s_t)$ from optimization problem (\ref{eq:barrier_control_guide})
  		\State Deploy controller $u_0(s_t) = u_{\theta_0}^{RL}(s_t) + u^{CBF}_0(s_t)$
  		\State Store state-action pair $(s_t, u_0^{CBF})$ in $\hat{A}$ 
  		\State Observe $(s_t, u_0, s_{t+1},r_t)$ and store in $\hat{D}$
  		\EndFor
  		\State Collect Episode Reward, $~ \sum_{t=1}^T r_t$
  		\State Update GP model using (\ref{eq:GP_model}) and measurements $\hat{D}$ 
  		\State Set $k=1$ (representing $k^{th}$ policy iteration)
  		\While {$k < ~ $Episodes}
  		\State{Do policy iteration using RL algorithm based on }
  		\Statex{~~~~~ previously observed episode/rewards to obtain $\pi_{\theta_k}^{RL}$}
  		\State Train $u_{\phi_k}^{bar}$ to approximate prior CBF controllers \Statex $~~~~~~~~~~~~~~~~~~~~~~~~~ ( u_{\phi_k}^{bar} = u_0^{CBF} + ...  u_{k-1}^{CBF})$ using $\hat{A}$
  		\State Initialize state $s_0 \sim \rho_0$
  		\For {$t = 1,\ldots,T$}
  		\State Sample control $u_{\theta_k}^{RL}(s_t) + u_{\phi_k}^{bar}(s_t)$
  		\State Solve for $u^{CBF}_k(s_t)$ from problem (\ref{eq:barrier_control_guide})
  		\State Deploy controller $u_k(s_t) = u_{\theta_k}^{RL}(s_t) $ 
  		\Statex $ ~~~~~~~~~~~~~~~~~~~~~~~~~~~~~~~~~~~~~~~~~~~~~~~ + u_{\phi_k}^{bar}(s_t) + u^{CBF}_k(s_t)$. 
  		\State Store state-action pair ($s_t,u_{\phi_k}^{bar} + u_k^{CBF}$) in $\hat{A}$
  		\State Observe $(s_t, u_k, s_{t+1}, r_t)$ and store in $\hat{D}$
  		\EndFor
  		\State Collect Episode Reward, $~ \sum_{t=1}^T r_t$
  		\State Update GP model using (\ref{eq:GP_model}) and measurements $\hat{D}$ 
  		\State $k = k + 1$
  		\EndWhile
  		\State \textbf{return} $\pi_{\theta_k}^{RL}, u_{\phi_k}^{bar}, u_k^{CBF}$ 
  		\Statex ~~~~~~~~~~~~~~~~~~~~~~~~~~~~~~~~~~~ $\rhd$ Overall controller composed
  		\Statex ~~~~~~~~~~~~~~~~~~~~~~~~~~~~~~~~~~~~~~~~ from all 3 subcomponents
  	\end{algorithmic}
  \end{algorithm}

\section{Experiments}

We implement two versions of the RL-CBF algorithm with existing model-free RL algorithms: TRPO-CBF, derived from TRPO \cite{Schulman2015}, and DDPG-CBF, derived from DDPG  \cite{Lillicrap2015}. The code for these examples can be found at: \url{https://github.com/rcheng805/RL-CBF}.

\subsection{Inverted Pendulum}

We first apply RL-CBF to the control of a simulated inverted pendulum from the OpenAI gym environment (\textit{pendulum-v0}), which has mass $m$ and length, $l$, and is actuated by torque, $u$. We set the safe region to be $\theta \in [-1, 1]$ radians, and define the reward function $r = \theta^2 + 0.1 \dot{\theta}^2 + 0.001 u^2$ to learn a controller that keeps the pendulum upright. The true system dynamics are defined as follows,
\begin{equation}
\begin{split}
& \theta_{t+1} = \theta_t + \dot{\theta_t} \delta t + \frac{3 g}{2 l} \sin(\theta_t) \delta t^2 + \frac{3}{m l^2} u \delta t^2 , \\
& \dot{\theta}_{t+1} = \dot{\theta_t} + \frac{3 g}{2 l} \sin(\theta_t) \delta t + \frac{3}{m l^2} u \delta t ,
\end{split}
\end{equation}
 
 \noindent
 with torque limits $u \in [-15, 15]$, and $m = 1, ~ l = 1$. To introduce model uncertainty, our nominal model assumes $m=1.4, ~ l=1.4$ ($40\%$ error in model parameters).

Figure \ref{fig:pendulum_reward} compares the accumulated reward achieved during each episode using TRPO, DDPG, TRPO-CBF, and DDPG-CBF. The two RL-CBF algorithms converge near the optimal solution very rapidly, and significantly outperform the corresponding baseline algorithms without the CBFs. We note that TRPO and DDPG \textit{sometimes} converge on a high-performance controller (comparable to TRPO-CBF and DDPG-CBF), though this occurs less reliably and more slowly, resulting in the poorer learning curves. More importantly, the RL-CBF controllers maintain safety (i.e. never leave the safe region) throughout the learning process, as also seen in Figure \ref{fig:pendulum_reward}. In contrast, TRPO and DDPG severely violate safety while learning the optimal policy.

Figure \ref{fig:pendulum_angle} shows the pendulum angle during a representative trial under the first policy versus the last learned policy deployed for TRPO-CBF and DDPG-CBF. For the first policy iteration, the pendulum angle is maintained near the edge of the safe region -- the RL algorithm has proposed a poor controller so the CBF controller takes the minimal action necessary to keep the system safe. By the last iteration though, the CBF controller is completely inactive ($u^{CBF} = 0$), since the guided RL controller ($u^{RL}_{\theta_k}(s) + u^{bar}_{\phi_k}(s)$) is already safe.

\begin{figure}[!h]
	\centering
	\includegraphics[width=0.95\columnwidth]{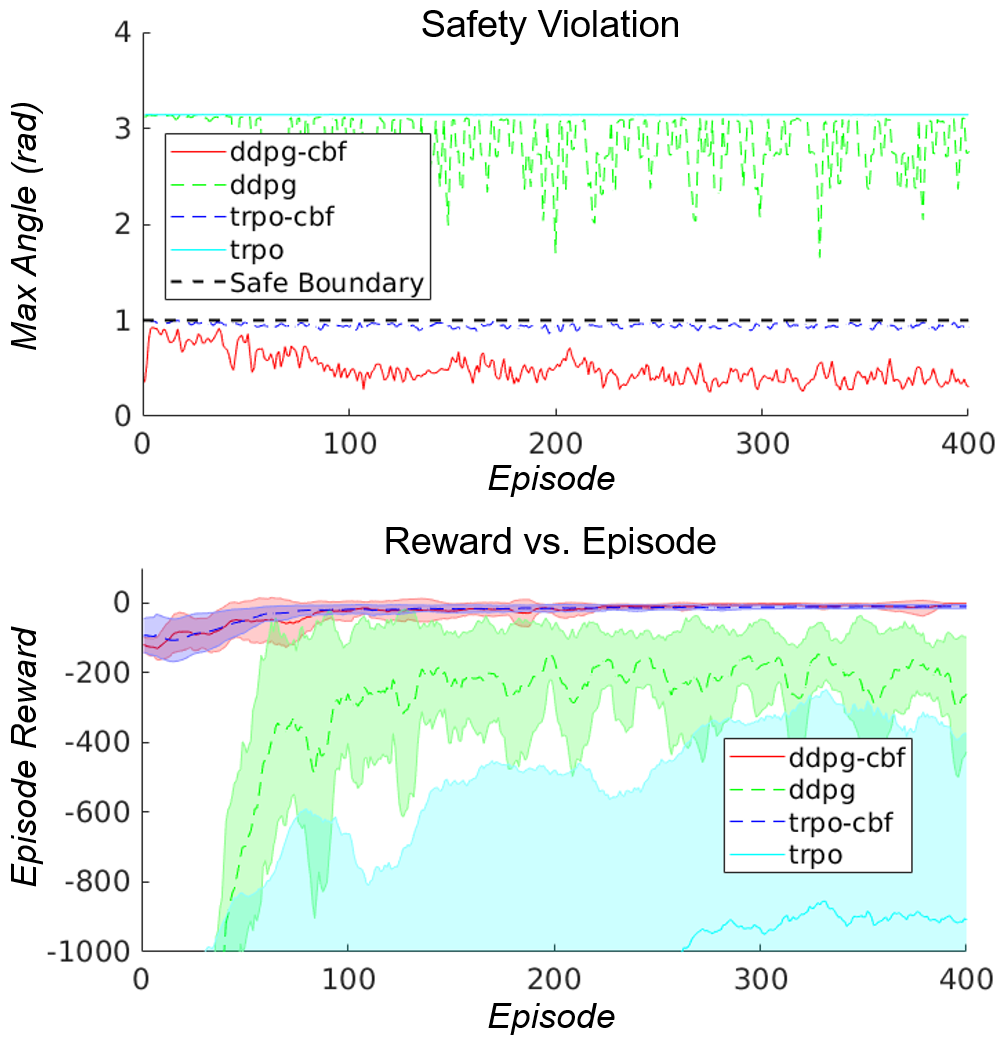}
	\caption{(Top) Maximum angle (rad) of the pendulum throughout each episode. Values above the dashed black line represent exits from the safe set at some point during the episode. (Bottom) Comparison of accumulated reward from inverted pendulum problem using TRPO, DDPG, TRPO-CBF, and DDPG-CBF.}
	\label{fig:pendulum_reward}
\end{figure}

\begin{figure}[!h]
	\includegraphics[width=0.95\columnwidth]{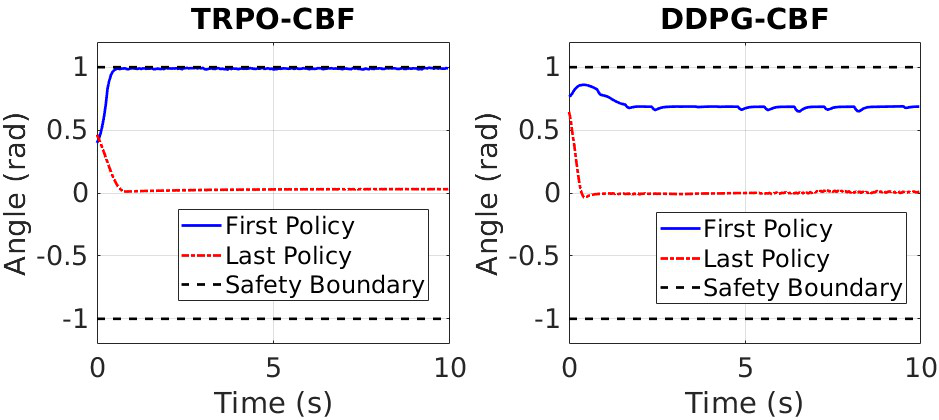}
	\caption{Representative pendulum trajectory (angle vs. time) using first policy vs last policy. The left plot and right plot show results from TRPO-CBF and DDPG-CBF, respectively. The trajectory for the first policy (blue) goes to edge of the safe region and stays there, while the trajectory for the last policy (red) quickly converges to the upright position.}
	\label{fig:pendulum_angle}
\end{figure}

\subsection{Simulated Car Following}

Consider a chain of five cars following each other on a straight road. We control the acceleration/deceleration of the $4^{th}$ car in the chain, and would like to train a policy to maximize fuel efficiency during traffic congestion while avoiding collisions. Each car utilizes the dynamics shown in equation (\ref{eq:car_dynamics}), and we attempt to optimize the reward function (\ref{eq:car_reward}). The car dynamics and reward function are inspired by previous work \cite{He2018}.
\begin{equation}
\begin{bmatrix}
\dot{s}^{(i)} \\
\dot{v}^{(i)} 
\end{bmatrix}
=
\begin{bmatrix}
0 & 1 \\
0 & -k_d 
\end{bmatrix}
\begin{bmatrix}
s^{(i)}  \\
v^{(i)}
\end{bmatrix}
+
\begin{bmatrix}
0 \\
1 
\end{bmatrix}
a
~~~~~~~ k_d = 0.1.
\label{eq:car_dynamics}
\end{equation}

\begin{equation}
\begin{split}
& r = -\sum_{t=1}^T \Big[ v^{(4)}_t \max((a^{(4)}_t), 0) + \sum_{i=3}^4 G_i \Big( \frac{500}{s_t^{(i)} - s_t^{(i+1)}} \Big) \Big] , \\
& ~ G_m(x) = \begin{cases}
\ |x| &\text{if $s^{(m)} - s^{(m+1)} \leq 3$}\\
0 &\text{otherwise} \\
\end{cases}
\label{eq:car_reward}
\end{split}
\end{equation}

The first term in the reward optimizes fuel efficiency, while the other term encourages the car to maintain a 3 meter distance from the other cars (soft constraint). For the RL-CBF controllers, the CBF enforces a 2 meter safe distance between cars (hard constraint). The behavior of cars 1,2,3, and 5 is described in the Appendix.

The $4^{th}$ car has access to every other cars' position, velocity, and acceleration, but it only has a crude model of its own dynamics ($k_d = 0$) and an inaccurate model of the drivers behind and in front of it. In addition, we add Gaussian noise to the acceleration of each car. The idea is that the $4^{th}$ car can use its crude model to guarantee safety with high probability, \textit{and} improve fuel efficiency by slowly building and leveraging an implicit model of the other drivers' behaviors.

From Figure \ref{fig:car_reward}, we see that there were no safety violations between the cars during our simulated experiments when using either of the RL-CBF controllers. When using TRPO and DDPG alone without CBF safety, almost all trials had collisions, even in the later stages of learning. Furthermore, as seen in Figure \ref{fig:car_reward}, TRPO-CBF learns faster and outperforms TRPO (DDPG-CBF also outperforms DDPG though neither algorithm converged on a high-performance controller in our experiments). It is important to note that in \textit{some} experiments, TRPO finds a comparable controller to TRPO-CBF, but this is often not the case due to randomness in seeds. 

Although DDPG and DDPG-CBF failed to converge on a good policy, Figure \ref{fig:car_reward} shows that DDPG-CBF (and TRPO-CBF) always maintained a safe controller. This is a crucial benefit of the RL-CBF approach, as it guarantees safety independent of the system's learning performance.

\begin{figure}[!h]
	\includegraphics[width=0.95\columnwidth]{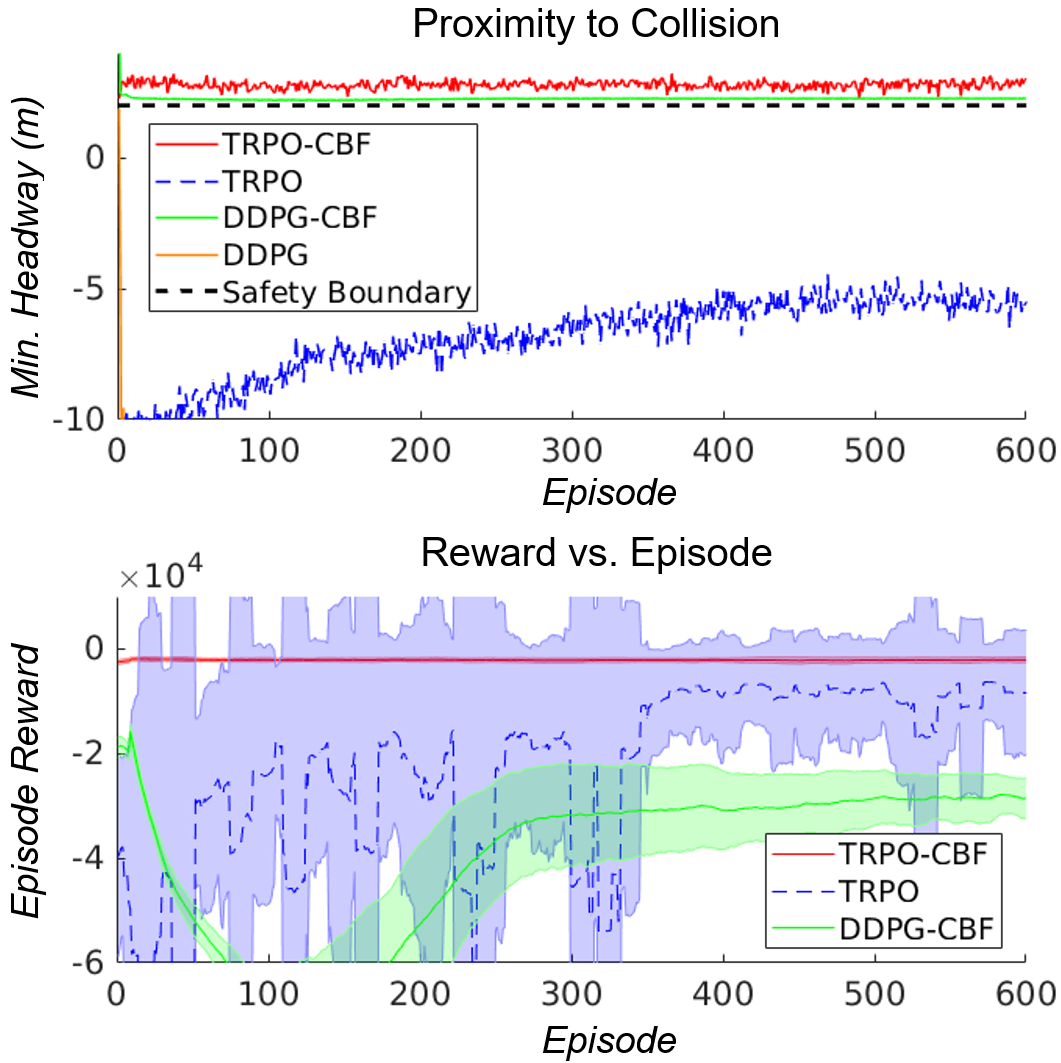}
	\caption{(Top) Minimum headway between cars during each learning episode using DDPG, TRPO, DDPG-CBF, and TRPO-CBF. Values below the dashed black line represent exits from the safe set, and values below 0 represent collisions. The curve for DDPG has high negative values throughout learning, and is not seen. (Bottom) Comparison of reward over multiple episodes from car-following problem using TRPO, TRPO-CBF, and DDPG-CBF (DDPG is excluded because it exhibits very poor performance).}
	\label{fig:car_reward}
\end{figure}

\section{Conclusion}

Adding even crude model information and CBFs into the model-free RL framework allows us to improve the exploration of model-free learning algorithms while ensuring end-to-end safety. Therefore, we proposed the safe RL-CBF framework, and developed an efficient controller synthesis algorithm that guarantees safety \textit{and} improves exploration. These features will be crucial in deploying reinforcement learning on physical systems, where problems require online computation and efficient learning with safety guarantees.

This framework, which combines model-free RL-based control, model-based CBF control, and model learning has the additional advantages of being able to (1) easily integrate new RL algorithms (in place of TRPO/DDPG) as they are developed, and (2) incorporate better model information from measurements to \textit{online} improve the CBF controller.

A significant assumption in this work is that we are given a valid safe set, $h(s)$, which can be rendered forward invariant. However, computing these valid safe sets is non-trivial and computationally intensive \cite{Wang2017,Wabersich2018,Fisac2018}. If we are \textit{not} given a valid safe set, we may reach states where it is not possible to remain safe (i.e. $\epsilon^{max} \geq 0$). Although our controller achieves graceful degradation in these cases, in future work it will be important to learn the safe set in addition to the controller.

\section*{Acknowledgment}
The authors would like to thank Hoang Le and Yisong Yue for helpful discussions.

\bibliographystyle{aaai}
\bibliography{references}

\begin{thebibliography}{}

\bibitem[\protect\citeauthoryear{Abbeel and Ng}{2004}]{Abbeel2004}
Abbeel, P., and Ng, A.~Y.
\newblock 2004.
\newblock {Apprenticeship learning via inverse reinforcement learning}.
\newblock In {\em Twenty-first international conference on Machine learning -
  ICML '04}.

\bibitem[\protect\citeauthoryear{Abbeel, Coates, and Ng}{2010}]{Abbeel2010}
Abbeel, P.; Coates, A.; and Ng, A.~Y.
\newblock 2010.
\newblock {Autonomous helicopter aerobatics through apprenticeship learning}.
\newblock {\em International Journal of Robotics Research}.

\bibitem[\protect\citeauthoryear{Achiam \bgroup et al\mbox.\egroup
  }{2017}]{Achiam2017}
Achiam, J.; Held, D.; Tamar, A.; and Abbeel, P.
\newblock 2017.
\newblock {Constrained Policy Optimization}.
\newblock {\em arXiv preprint arXiv:1705:10528}.

\bibitem[\protect\citeauthoryear{Agrawal and Sreenath}{2017}]{agrawal2017}
Agrawal, A., and Sreenath, K.
\newblock 2017.
\newblock {Discrete Control Barrier Functions for Safety-Critical Control of
  Discrete Systems with Application to Bipedal Robot Navigation}.
\newblock {\em Robotics science and systems (RSS)}.

\bibitem[\protect\citeauthoryear{Alshiekh \bgroup et al\mbox.\egroup
  }{2017}]{Alshiekh2017}
Alshiekh, M.; Bloem, R.; Ehlers, R.; K{\"{o}}nighofer, B.; Niekum, S.; and
  Topcu, U.
\newblock 2017.
\newblock {Safe Reinforcement Learning via Shielding}.
\newblock {\em arXiv preprint arXiv:1708.08611}.

\bibitem[\protect\citeauthoryear{Ames \bgroup et al\mbox.\egroup
  }{2017}]{ames2017}
Ames, A.~D.; Xu, X.; Grizzle, J.~W.; and Tabuada, P.
\newblock 2017.
\newblock {Control Barrier Function Based Quadratic Programs for Safety
  Critical Systems}.
\newblock {\em IEEE Transactions on Automatic Control}.

\bibitem[\protect\citeauthoryear{Berkenkamp \bgroup et al\mbox.\egroup
  }{2017}]{Berkenkamp2017}
Berkenkamp, F.; Turchetta, M.; Schoellig, A.~P.; and Krause, A.
\newblock 2017.
\newblock {Safe Model-based Reinforcement Learning with Stability Guarantees}.
\newblock In {\em Neural Information Processing Systems}.

\bibitem[\protect\citeauthoryear{Bertsekas}{2005}]{Bertsekas2005}
Bertsekas, D.
\newblock 2005.
\newblock {\em {Dynamic Programming and Optimal Control}}.

\bibitem[\protect\citeauthoryear{Chow \bgroup et al\mbox.\egroup
  }{2018}]{Chow2018}
Chow, Y.; Nachum, O.; Duenez-Guzman, E.; and Ghavamzadeh, M.
\newblock 2018.
\newblock {A Lyapunov-based Approach to Safe Reinforcement Learning}.
\newblock {\em arXiv preprint arXiv:1805.07708}.

\bibitem[\protect\citeauthoryear{Duan \bgroup et al\mbox.\egroup
  }{2016}]{Duan2016}
Duan, Y.; Chen, X.; Schulman, J.; and Abbeel, P.
\newblock 2016.
\newblock {Benchmarking Deep Reinforcement Learning for Continuous Control}.
\newblock {\em arXiv}.

\bibitem[\protect\citeauthoryear{Fisac \bgroup et al\mbox.\egroup
  }{2018}]{Fisac2018}
Fisac, J.~F.; Akametalu, A.~K.; Zeilinger, M.~N.; Kaynama, S.; Gillula, J.; and
  Tomlin, C.~J.
\newblock 2018.
\newblock {A General Safety Framework for Learning-Based Control in Uncertain
  Robotic Systems}.
\newblock {\em arXiv preprint arXiv:1705.01292}.

\bibitem[\protect\citeauthoryear{Fu, Glover, and April}{2005}]{Fu2005}
Fu, M.; Glover, F.; and April, J.
\newblock 2005.
\newblock {Simulation optimization: a review, new developments, and
  applications}.
\newblock {\em Proceedings of the Winter Simulation Conference, 2005.}

\bibitem[\protect\citeauthoryear{Garc{\'{i}}a and
  Fern{\'{a}}ndez}{2015}]{Garcia2015}
Garc{\'{i}}a, J., and Fern{\'{a}}ndez, F.
\newblock 2015.
\newblock {A Comprehensive Survey on Safe Reinforcement Learning}.
\newblock {\em Journal of Machine Learning Research}.

\bibitem[\protect\citeauthoryear{Gaskett}{2003}]{Gaskett2003a}
Gaskett, C.
\newblock 2003.
\newblock {Reinforcement Learning in Circumstances Beyond its Control}.
\newblock In {\em CIMCA}.

\bibitem[\protect\citeauthoryear{Gillula and Tomlin}{2012}]{Gillula2012}
Gillula, J.~H., and Tomlin, C.~J.
\newblock 2012.
\newblock {Guaranteed safe online learning via reachability: Tracking a ground
  target using a quadrotor}.
\newblock In {\em Proceedings - IEEE International Conference on Robotics and
  Automation}.

\bibitem[\protect\citeauthoryear{He, Ge, and Orosz}{2018}]{He2018}
He, C.~R.; Ge, J.~I.; and Orosz, G.
\newblock 2018.
\newblock {Data-based fuel-economy optimization of connected automated trucks
  in traffic}.
\newblock {\em Annual American Control Conference (ACC)}.

\bibitem[\protect\citeauthoryear{Koller \bgroup et al\mbox.\egroup
  }{2018}]{Koller2018}
Koller, T.; Berkenkamp, F.; Turchetta, M.; and Krause, A.
\newblock 2018.
\newblock {Learning-based Model Predictive Control for Safe Exploration and
  Reinforcement Learning}.
\newblock {\em arXiv preprint arXiv:1803.08287}.

\bibitem[\protect\citeauthoryear{Li, Kalabic, and Chu}{2018}]{Li2018}
Li, Z.; Kalabic, U.; and Chu, T.
\newblock 2018.
\newblock {Safe Reinforcement Learning: Learning with Supervision Using a
  Constraint-Admissible Set}.
\newblock In {\em Annual American Control Conference}.

\bibitem[\protect\citeauthoryear{Lillicrap \bgroup et al\mbox.\egroup
  }{2015}]{Lillicrap2015}
Lillicrap, T.~P.; Hunt, J.~J.; Pritzel, A.; Heess, N.; Erez, T.; Tassa, Y.;
  Silver, D.; and Wierstra, D.
\newblock 2015.
\newblock {Continuous control with deep reinforcement learning}.
\newblock {\em arXiv preprint arXiv:1509.02971}.

\bibitem[\protect\citeauthoryear{Mannucci \bgroup et al\mbox.\egroup
  }{2018}]{Mannucci2018}
Mannucci, T.; {Van Kampen}, E.~J.; {De Visser}, C.; and Chu, Q.
\newblock 2018.
\newblock {Safe Exploration Algorithms for Reinforcement Learning Controllers}.
\newblock {\em IEEE Transactions on Neural Networks and Learning Systems}.

\bibitem[\protect\citeauthoryear{Moldovan and Abbeel}{2012}]{Moldovan2012}
Moldovan, T.~M., and Abbeel, P.
\newblock 2012.
\newblock {Safe Exploration in Markov Decision Processes}.
\newblock {\em arXiv preprint arXiv:1205.4810}.

\bibitem[\protect\citeauthoryear{Nguyen-Tuong, Seeger, and
  Peters}{2009}]{Nguyen-Tuong2009}
Nguyen-Tuong, D.; Seeger, M.; and Peters, J.
\newblock 2009.
\newblock {Local Gaussian Process Regression for Real Time Online Model
  Learning and Control}.
\newblock In {\em Advances in neural information processing systems}.

\bibitem[\protect\citeauthoryear{Ohnishi \bgroup et al\mbox.\egroup
  }{2018}]{Ohnishi2018}
Ohnishi, M.; Wang, L.; Notomista, G.; and Egerstedt, M.
\newblock 2018.
\newblock {Safety-aware Adaptive Reinforcement Learning with Applications to
  Brushbot Navigation}.
\newblock {\em arXiv preprint arXiv:1801.09627}.

\bibitem[\protect\citeauthoryear{Perkins and Barto}{2003}]{Perkins2003}
Perkins, T.~J., and Barto, A.~G.
\newblock 2003.
\newblock {Lyapunov design for safe reinforcement learning}.
\newblock {\em Journal of Machine Learning Research}.

\bibitem[\protect\citeauthoryear{Peters and Schaal}{2008}]{Peters2008}
Peters, J., and Schaal, S.
\newblock 2008.
\newblock {Reinforcement learning of motor skills with policy gradients}.
\newblock {\em Neural Networks}.

\bibitem[\protect\citeauthoryear{Rasmussen and Williams}{2006}]{Rasmussen2006}
Rasmussen, C.~E., and Williams, C.~K.
\newblock 2006.
\newblock {\em {Gaussian Processes for Machine Learning}}.

\bibitem[\protect\citeauthoryear{Schulman \bgroup et al\mbox.\egroup
  }{2015}]{Schulman2015}
Schulman, J.; Levine, S.; Moritz, P.; Jordan, M.; and Abbeel, P.
\newblock 2015.
\newblock {Trust Region Policy Optimization}.
\newblock In {\em International Conference on Machine Learning (ICML)}.

\bibitem[\protect\citeauthoryear{Silver \bgroup et al\mbox.\egroup
  }{2014}]{Silver2014}
Silver, D.; Lever, G.; Heess, N.; Degris, T.; Wierstra, D.; and Riedmiller, M.
\newblock 2014.
\newblock {Deterministic Policy Gradient Algorithms}.
\newblock {\em Proceedings of the 31st International Conference on Machine
  Learning (ICML-14)}.

\bibitem[\protect\citeauthoryear{Snelson and Ghahramani}{2007}]{Snelson2007}
Snelson, E., and Ghahramani, Z.
\newblock 2007.
\newblock {Local and global sparse Gaussian process approximations}.
\newblock {\em Proceedings of the International Conference on Artificial
  Intelligence and Statistics (AISTATS)}.

\bibitem[\protect\citeauthoryear{Tang \bgroup et al\mbox.\egroup
  }{2010}]{Tang2010}
Tang, J.; Singh, A.; Goehausen, N.; and Abbeel, P.
\newblock 2010.
\newblock {Parameterized maneuver learning for autonomous helicopter flight}.
\newblock In {\em Proceedings - IEEE International Conference on Robotics and
  Automation}.

\bibitem[\protect\citeauthoryear{Wabersich and Zeilinger}{2018}]{Wabersich2018}
Wabersich, K.~P., and Zeilinger, M.~N.
\newblock 2018.
\newblock {Scalable synthesis of safety certificates from data with
  applications to learning-based control}.
\newblock {\em arXiv preprint arXiv:1711.11417}.

\bibitem[\protect\citeauthoryear{Wachi \bgroup et al\mbox.\egroup
  }{2018}]{Wachi2018}
Wachi, A.; Sui, Y.; Yue, Y.; and Ono, M.
\newblock 2018.
\newblock {Safe Exploration and Optimization of Constrained MDPs using Gaussian
  Processes}.
\newblock {\em 32nd AAAI conference on Artificial Intelligence (AAAI)}.

\bibitem[\protect\citeauthoryear{Wang, Theodorou, and
  Egerstedt}{2017}]{Wang2017}
Wang, L.; Theodorou, E.~A.; and Egerstedt, M.
\newblock 2017.
\newblock {Safe Learning of Quadrotor Dynamics Using Barrier Certificates}.
\newblock {\em arXiv preprint arXiv:1710:05472}.

\end{thebibliography}

\pagebreak

\begin{appendices}

	\section{Appendix A: Proof of Theorem 2}
	
	\addtocounter{theorem}{-1}
	\addtocounter{equation}{0}
	
	\begin{theorem}
		Using the control law $u_k(s)$ from (15), if there exists a solution to problem (16) such that $\epsilon^{max} = 0$, then the safe set $\mathcal{C}$ is forward invariant with probability $(1-\delta)$. If $\epsilon^{max} > 0$, but the solution to problem (16) satisfies $\epsilon \leq \epsilon^{max}$ for all $s \in \mathcal{C}_{\epsilon}$, then the controller will render the set $\mathcal{C}_{\epsilon}$ forward invariant with probability $(1-\delta)$.
		
		Furthermore, if we use TRPO for the RL algorithm, then the control law $u_k^{prop}(s) = u_k(s) - u_k^{CBF}(s)$ from (15) achieves the performance guarantee $J({\pi_{k}^{prop}}) \geq J({\pi_{k-1}}) - \frac{2 \lambda \gamma}{(1-\gamma)^2} \delta_{\pi}$, where $\lambda = \max_s | \mathbb{E}_{a \sim \pi_{k}^{prop}} [ A_{\pi_{k-1}}(s,a)]|$ and $\delta_{\pi}$ is chosen as in equation (4).
		\label{theorem:trpo_cbf}
	\end{theorem}
	
	\begin{proof}	
		To prove the performance bound in the second part of the theorem, we use the property of the advantage function from equation (\ref{eq:trpo_1}) below:
		
		\begin{equation}
		\begin{split}
		& J(\pi_{k}) = J(\pi_{k-1}) + \mathbb{E}_{\tau \sim \pi_{k}} \Big[ \sum_{t=0}^{\infty} \gamma^t A_{\pi_{k-1}}(s_t,a_t) \Big],
		\end{split}
		\label{eq:trpo_1}
		\end{equation}
		
		\noindent
		where $s_{t+1} \sim P(s_{t+1} | s_t, a_t)$.  As derived in (Schulman et al. 2015), we can then obtain the following inequality:

		\begin{equation}
		\begin{split}
		& J(\pi_{k}) \geq J(\pi_{k-1}) + \frac{1}{1-\gamma} \mathbb{E}_{\substack{s_t \sim \pi_{k-1} \\ a_t \sim \pi_{k}}} \Big[ \sum_{t=0}^{\infty} \gamma^t A_{\pi_{k-1}}(s_t,a_t) \\
		& ~~~~~~~~~~~~~~~~~~~~~~~~~~~~ - \frac{2 \gamma \lambda}{1-\gamma} D_{TV} (\pi_{k-1}, \pi_{k})\Big],
		\end{split}
		\label{eq:trpo_proof_1}
		\end{equation}
		
		\noindent
		where $D_{TV}(\pi_{k-1}, \pi_{k})$ is the total variational distance between policies $\pi_{k-1}$ and $\pi_{k}$, and $\lambda = \max_s | \mathbb{E}_{a \sim \pi_{k}} [ A_{\pi_{k-1}}(s,a)]|$. Note that our CBF controllers are all deterministic, so we can redefine $u_{k-1}^{barrier} = \sum_{j=0}^{k-2} u_j^{CBF} + u_{k-1}^{CBF} = \sum_{j=0}^{k-1} u_j^{CBF}$. Based on this definition and equation (15), we can rewrite/define the following controllers:
		
		\begin{equation}
		\begin{split}
		& u_{k-1}(s) = u_{\theta_{k-1}}^{RL}(s) + u_{k-1}^{barrier}(s), \\
		& \pi_{k-1}(a | s) = \pi_{\theta_{k-1}}^{RL} (a - u_{k-1}^{barrier}(s) ~ | ~ s ), \\
		\end{split}
		\end{equation}
		
		\begin{equation}
		\begin{split}
		& u_{k}^{prop}(s) = u^{RL}_{\theta_{k}}(s) + u_{k-1}^{barrier}(s), \\
		& \pi_{k}^{prop}(a | s) = \pi_{\theta_{k}}^{RL} (a - u_{k-1}^{barrier}(s) ~ | ~ s). \\
		\end{split}
		\label{eq:stochastic_control}
		\end{equation}

		We can plug in the above relations for $\pi_{k-1}$ and $\pi^{prop}_{k}$ into inequality (\ref{eq:trpo_proof_1}), to obtain the following bound (we plug in $\pi^{prop}_{k}$ for $\pi_{k}$):
		
		\begin{equation}
		\begin{split}
		& J(\pi_{k}^{prop}) \geq J(\pi_{k-1}) + \frac{1}{1-\gamma} \mathbb{E}_{\substack{s_t \sim \pi_{k-1} \\ a_t \sim \pi_{k}^{prop}}} \Big[ \sum_{t=0}^{\infty} \gamma^t A_{\pi_{k-1}}(s_t,a_t) \\
		& - \frac{2 \gamma \lambda}{1-\gamma} D_{TV} (\pi_{\theta_{k-1}}^{RL} (a - u_{k-1}^{barrier}), \pi_{\theta_{k}}^{RL} (a - u_{k-1}^{barrier} ))\Big],
		\end{split}
		\end{equation}
		
		\noindent
		where we drop the policies' dependency on the state $s$ for compactness. Due to the shift invariance of the total variational distance, $D_{TV}$, we can simplify this to:  
		
		\begin{equation}
		\begin{split}
		& J(\pi_{k}^{prop}) \geq J(\pi_{k-1}) + \frac{1}{1-\gamma} \mathbb{E}_{\substack{s_t \sim \pi_{k-1} \\ a_t \sim \pi_{k}^{prop}}} \Big[ \sum_{t=0}^{\infty} \gamma^t A_{\pi_{k-1}}(s_t,a_t) \\
		& ~~~~~~~ - \frac{2 \gamma \lambda}{1-\gamma} D_{TV} (\pi_{\theta_{k-1}}^{RL} , \pi_{\theta_{k}}^{RL})\Big]. \\
		\end{split}
		\label{eq:trpo_proof_2}
		\end{equation}
		
		Because $\pi_{k-1}$ is a feasible point of the TRPO optimization problem (4) with objective value 0, we know that our solution $\pi_{k}^{prop}$ satisfies the following:
		
		$$\mathbb{E}_{\substack{s_t \sim \pi_{k-1} \\ a_t \sim \pi_{k}^{prop}}} \Big[ \sum_{t=0}^{\infty} \gamma^t A_{\pi_{k-1}}(s_t,a_t) \Big] \geq 0.$$
		
		\noindent
		Since the optimization problem (4) specifies the bound $D_{TV}(\pi_{\theta_{k-1}}^{RL}, \pi_{\theta_{k}}^{RL}) \leq \delta_{\pi}$, then it follows that:
		
		\begin{equation}
		J({\pi_{k}^{prop}}) \geq J({\pi_{k-1}}) - \frac{2 \lambda \gamma}{(1-\gamma)^2} \delta_{\pi},
		\label{eq:perf_bound}
		\end{equation}
		
		\noindent
		where $\lambda = \max_s | \mathbb{E}_{a \sim \pi_{k}^{prop}} [ A_{\pi_{k-1}}(s,a)]|$. The realization of the policy $\pi_{k}^{prop}(a|s)$ is:
		
		$$u_{k}^{prop}(s) = u_{\theta_{k}}^{RL}(s) + u_{k-1}^{barrier}(s) = u_k(s) - u_k^{CBF}(s).$$
		
		Therefore, if we utilize the policy $u_k(s) - u_k^{CBF}(s)$, we can obtain the performance bound in equation (\ref{eq:perf_bound}).

	\end{proof}
	
	~
	
	~
	
	\section{Appendix B: Car-Following Problem}
	
	\subsection{Driver Behavior and System Dynamics}
	
	In this section, we elaborate on the behavior of the cars in the car-following numerical experiment. The dynamics for the drivers follows equation (20), and their acceleration is described as follows:
	
	\begin{equation}
	\begin{split}
	& a^{(1)} = v_{des} - 10 \sin(0.2 t) \\
	& a^{(i)} = k_p (v_{des} - v^{(i)}) - k_b G_1(s^{(i-1)} - s^{(i)}) \textnormal{ for } i = 2, 3 \\
	& a^{(5)} = k_p (v_{des} - v^{(i)}) - \frac{1}{2} k_b G_2(s^{(3)} - s^{(5)}) \textnormal{ for } i = 5 \\
	& ~ G_1(x) = \begin{cases}
	x &\text{if $x \leq 6$}\\
	0 &\text{otherwise} \\
	\end{cases} ,
	~ G_2(x) = \begin{cases}
	x &\text{if $x \leq 12$}\\
	0 &\text{otherwise} \\
	\end{cases} \\
	& k_p = 4, ~~ k_b = 20, ~~ v_{des} = 30, ~~~ a \in [-100, 100]
	\end{split}
	\end{equation}
	
	\noindent
	where $a^{(i)}$ represents the acceleration for driver $i$. In addition, gaussian noise is added to the acceleration of each driver. In driver four's nominal model of the other drivers' behavior, $k_p = 3.5$, $k_b = 18$, and $k_d = 0$. 
	
	\subsection{Explanation for High Reward of DDPG in Initial Trials}
	
	In Figure 5, the reward of DDPG-CBF starts very high for early trials, and then drops to lower values. This arises due to stochasticity in the drivers' behaviors, which makes certain bad control strategies perform well in rare specific cases. 
	
	In \textit{most} trials, our car must accelerate at certain points (decreasing reward) in order to avoid collision with the driver behind. However, if the rear driver significantly slows down during certain trials due to stochasticity in their behavior, our car can simply cruise with little acceleration throughout these trials (these correspond to the few, initial high reward trials). 
	
	This strategy of cruising (little/no acceleration) is generally bad because if the driver behind does not slow down, our car must accelerate heavily at the last second to avoid collision, accumulating heavy penalty. The DDPG algorithm learns to avoid this “do nothing initially” strategy.
	
\end{appendices}

\end{document}